\documentclass[conference]{IEEEtran}
\usepackage{amsmath,amssymb,amsfonts}
\usepackage{bm}
\usepackage{cite}
\usepackage{algorithm}
\usepackage{algorithmic}
\usepackage{graphicx}
\usepackage{url}
\usepackage{amsthm}
\usepackage{changes}

\newtheorem{definition}{Definition}
\newtheorem{theorem}{Theorem}
\newtheorem{lemma}{Lemma}
\newtheorem{remark}{Remark}

\newcommand{\R}{\mathbb{R}}
\newcommand{\E}{\mathbb{E}}
\newcommand{\cA}{\mathcal{A}}
\newcommand{\cB}{\mathcal{B}}
\newcommand{\ip}[2]{\left\langle #1,\,#2\right\rangle}
\newcommand{\norm}[1]{\left\lVert #1\right\rVert}

\title{Distributed Perceptron under Bounded Staleness, Partial Participation, and Noisy Communication}

\author{
\IEEEauthorblockN{Keval Jain}
\IEEEauthorblockA{\textit{IIIT Hyderabad}\\
keval.jain@research.iiit.ac.in}
\and
\IEEEauthorblockN{Anant Raj}
\IEEEauthorblockA{\textit{Indian Institute of Science}\\
anantraj@iisc.ac.in}
\and
\IEEEauthorblockN{Saurav Prakash}
\IEEEauthorblockA{\textit{IIT Madras}\\
saurav@ee.iitm.ac.in}
\and
\IEEEauthorblockN{Girish Varma}
\IEEEauthorblockA{\textit{IIIT Hyderabad}\\
girish.varma@iiit.ac.in}
}

\begin{document}
\maketitle

\begin{abstract}
We study a semi-asynchronous client-server perceptron trained via iterative parameter mixing (IPM-style averaging): clients run local perceptron updates and a server forms a global model by aggregating the updates that arrive in each communication round. The setting captures three system effects in distributed deployments: (i) \emph{stale updates} due to delayed model delivery and delayed application of client computations (two-sided version lag), (ii) \emph{partial participation} (intermittent client availability), and (iii) imperfect communication on both downlink and uplink, modeled as effective zero-mean additive noise with bounded second moment.
We introduce a server-side aggregation rule called \emph{staleness-bucket aggregation with padding} that deterministically enforces a prescribed \emph{staleness profile} over update ages without assuming any stochastic model for delays or participation.
Under margin separability and bounded data radius, we prove a finite-horizon expected bound on the cumulative \emph{weighted} number of perceptron mistakes over a given number of server rounds: the impact of delay appears only through the \emph{mean} enforced staleness, whereas communication noise contributes an additional term that grows on the order of the square root of the horizon with the total noise energy. In the noiseless case, we show how a finite expected mistake budget yields an explicit finite-round stabilization bound under a mild fresh-participation condition. Synthetic experiments match the predicted square root of horizon noise scaling. 
\end{abstract}


\section{Introduction}

The perceptron is the classical finite-mistake model for linearly separable classification. Its appeal is that the update rule is elementary, but the guarantee is exact in the geometry of the data: for examples of radius \(R\) separated by margin \(\gamma\), the sequential perceptron makes at most \(R^2/\gamma^2\) updates \cite{rosenblatt1958,novikoff1962}. This radius--margin law is the clean structure we build on. It assumes, however, a single model updated sequentially after each mistake. Distributed perceptron/iterative parameter mixing (IPM) methods let workers train locally and periodically average their models at a server \cite{mann2009,mcdonald2010}; the same local-update/model-averaging template underlies local stochastic gradient descent (local SGD) and Federated Averaging \cite{stich2019,mcmahan2017}. We ask how the Novikoff/IPM mistake law changes when the server model is formed from stale, intermittent, and noisy client updates, and whether eventual correctness of the server classifier can still be guaranteed in the noiseless/liveness regime. Keeping the analysis in finite-mistake terms lets the final bound retain the classical \(R,\gamma\) dependence while making the roles of staleness and communication noise explicit.

The surrounding literature is extensive, but usually optimizes different quantities. Bounded-staleness systems such as stale synchronous parallel (SSP) control loose consistency in distributed ML \cite{ho2013,cui2014}, and asynchronous optimization/SGD analyses give convergence or stationarity guarantees for stale gradient/coordinate updates \cite{peng2016,koloskova2022,mishchenko2022}. Asynchronous FL methods, including staleness-weighted or buffered aggregation, address stragglers and partial participation for FedAvg-type training \cite{xie2019,nguyen2022}. These works are broader in objective class, but their update object is typically a gradient or FedAvg-style optimization step, and their guarantees are convergence, stationarity, or optimization error. In an IPM perceptron, an arriving message is a locally trained model vector obtained after mistake-triggered updates from a possibly old server model, and it is then mixed into the current server classifier. Even though the perceptron update is related to subgradient descent on a hinge-type loss, a smooth-objective convergence theorem does not give a Novikoff-style count of client mistakes or a round after which the averaged classifier remains correct.

Delayed-feedback online learning is also adjacent, but the delay is used differently. Black-box reductions and adversarial-delay analyses handle late labels, losses, or gradients for a single online learner, often via feedback buffering or delayed online-gradient methods \cite{joulani2013,mcmahan2014,quanrud2015}. Here delay is stale client computation that the server deliberately averages, not feedback hidden from a sequential learner. Communication-impaired FL is another neighboring line. Some uplink/downlink FL work studies finite-precision transmission, where model coordinates are quantized, i.e., rounded to a limited number of bits \cite{zheng2021}; our additive model is for wireless/noisy-channel errors, not compression. In over-the-air FL, model vectors or updates are sent over broadcast or multiple-access links, and after power control, beamforming, or normalization the received vector is naturally represented as the desired model/aggregate plus an effective channel error \cite{yang2020,amiri2020,wei2022}. This effective-noise abstraction is used explicitly for both downlink and uplink noisy-channel FL, capturing channel noise and processing errors such as imperfect channel estimation or detection, with AWGN as a simple special case \cite{wei2022}. We assume only conditional zero mean and bounded second moment, not Gaussianity; the resulting guarantee is different: an explicit mistake budget under stale model averaging.

\textbf{Our Contributions.}
\begin{itemize}
\setlength{\itemsep}{0pt}
\setlength{\topsep}{1pt}
\setlength{\parsep}{0pt}
\setlength{\partopsep}{0pt}
\item We formulate a client-server IPM perceptron with arbitrary arrival sets, bounded two-sided staleness, and additive downlink/uplink noise.
\item We introduce \emph{staleness-bucket aggregation with padding}: each age bucket \(s\) receives prescribed mass \(\alpha_s\), while empty buckets are padded by cached iterates \(w_{t-s}\), so every server update has age profile \(\bm{\alpha}\) without a stochastic delay/participation model.
\item We prove the finite-horizon weighted mistake bound
\[
\E[K_A]\le \frac{S R^2}{\gamma^2}+\frac{\sqrt{S A V}}{\gamma},\qquad
S=1+\bar{s},\quad \bar{s}=\sum_{s=0}^{\tau}s\alpha_s,
\]
where \(V\) is the per-round effective noise-energy bound, recovering the classical and synchronous IPM perceptron bounds when staleness and noise are removed.
\item In the noiseless case, under a fresh-participation/liveness condition, we convert the finite mistake budget into explicit expected hitting and stabilization time bounds for a permanently correct server classifier.
\end{itemize}

\section{System Model and Staleness-Profile IPM Perceptron}

%\paragraph*{Notation}
Server rounds are $t=0,1,2,\ldots$ with server model $w_t\in\R^D$.
An update applied at round $t$ may have been computed from a stale server model $w_{t-s}$, where $s$ is its \emph{total} staleness (version lag).

% \subsection{Data and separability}

We consider binary classification over $\R^D$ with labels $y\in\{-1,+1\}$. There are $m$ clients indexed by $i\in\{1,\ldots,m\}$. Client $i$ holds a finite multiset of examples $T_i\subseteq \R^D\times\{-1,+1\}$, and the global dataset is $T=\biguplus_{i=1}^m T_i$.

\begin{definition}[Margin separability and radius]
The dataset $T$ is $\gamma$-margin separable with radius $R$ if there exist $w^\star\in\R^D$ with $\norm{w^\star}=1$ and constants $\gamma>0$, $R>0$ such that
\begin{equation}
 y\,\ip{w^\star}{x}\ge \gamma,\quad \forall (x,y)\in T.
 \label{eq:separable}
\end{equation}
with $\norm{x}\le R$ for all $(x,y)\in T$.
\end{definition}

\subsection{Two-sided staleness and server-enforced staleness profile}

The server maintains $w_t\in\R^D$ with initialization $w_0=0$ and the convention $w_s\equiv 0$ for all integers $s<0$.
% \paragraph{Arrival (participation) sets}
In each server round $t$, the server aggregates the set of client updates that \emph{arrive and are applied} at round $t$. We denote this (possibly history-dependent) arrival set by $\cA_t\subseteq\{1,\ldots,m\}$. No probabilistic model is imposed on the process $\{\cA_t\}_{t\ge 0}$.

\paragraph{Two-sided bounded staleness}
Fix bounds $\tau_{\mathrm{dl}},\tau_{\mathrm{ul}}\in\mathbb{N}_0$ and define the overall lag tolerance
\begin{equation}
\tau := \tau_{\mathrm{dl}}+\tau_{\mathrm{ul}}.
\label{eq:tau-total}
\end{equation}
For each arriving update $i\in\cA_t$, we associate:
(i) a \emph{download (downlink) staleness} $s^{\mathrm{dl}}_{i,t}\in\{0,1,\ldots,\tau_{\mathrm{dl}}\}$, meaning the client starts from a server iterate that is $s^{\mathrm{dl}}_{i,t}$ rounds old at its local start time; and
(ii) a \emph{turnaround delay} $s^{\mathrm{ul}}_{i,t}\in\{0,1,\ldots,\tau_{\mathrm{ul}}\}$, capturing computation plus uplink latency until the update is applied at the server.
Define the \emph{total staleness (version lag) at application time}
\begin{equation}
 s_{i,t} := s^{\mathrm{dl}}_{i,t}+s^{\mathrm{ul}}_{i,t}\in\{0,1,\ldots,\tau\}.
 \label{eq:total-staleness}
\end{equation}
Equivalently, if the client began local computation in round $a=t-s^{\mathrm{ul}}_{i,t}$ and read the stale model $w_{a-s^{\mathrm{dl}}_{i,t}}$, then
\begin{equation}
 w_{a-s^{\mathrm{dl}}_{i,t}} = w_{(t-s^{\mathrm{ul}}_{i,t})-s^{\mathrm{dl}}_{i,t}} = w_{t-s_{i,t}}.
 \label{eq:reindex}
\end{equation}
Thus the analysis depends only on the total lag $s_{i,t}$.

\paragraph{Client-side staleness and downlink noise}
Each arriving client update $i\in\cA_t$ is produced by running the perceptron on $T_i$ starting from a noisy stale model
\begin{equation}
 \tilde w^{\mathrm{dl}}_{i,t} = w_{t-s_{i,t}} + \delta_{i,t},
 \label{eq:downlink}
\end{equation}
where $\delta_{i,t}$ is an effective downlink noise vector.
Client $i$ runs the perceptron (any order; for at least one epoch), performs $k_{i,t}$ mistake updates, and outputs a local model $w_{i,t}$.

\paragraph{Uplink noise}
The server receives a noisy version of the returned model,
\begin{equation}
 \tilde w^{\mathrm{ul}}_{i,t} = w_{i,t} + \xi_{i,t},
 \label{eq:uplink}
\end{equation}
where $\xi_{i,t}$ is an effective uplink noise vector.

\paragraph{Server-enforced staleness profile via padding}
Fix a target \emph{staleness profile} $\bm{\alpha}=(\alpha_0,\ldots,\alpha_{\tau})$ with $\alpha_s\ge 0$ and $\sum_{s=0}^{\tau} \alpha_s=1$. Intuitively, this profile is the age distribution the server wants each aggregate to have: $\alpha_s$ is the intended fraction of aggregation mass assigned to updates whose base model is $s$ rounds old.
After observing the realized total stalenesses $\{s_{i,t}\}_{i\in\cA_t}$, define staleness buckets
\begin{equation}
 \cB_{s,t} := \{i\in\cA_t:\ s_{i,t}=s\},\qquad s=0,1,\ldots,\tau.
 \label{eq:buckets}
\end{equation}
The server assigns weights as follows:
\begin{itemize}
 \item If $\cB_{s,t}\neq\emptyset$, assign nonnegative weights $\{\mu_{i,t}\}_{i\in\cB_{s,t}}$ such that $\sum_{i\in\cB_{s,t}}\mu_{i,t}=\alpha_s$ (e.g., uniform: $\mu_{i,t}=\alpha_s/|\cB_{s,t}|$).
 \item If $\cB_{s,t}=\emptyset$, set a padding weight $\pi_{s,t}:=\alpha_s$, assigned to the cached server iterate $w_{t-s}$ (equivalently, a ``virtual participant'' at staleness $s$ returning $w_{t-s}$ and making $0$ mistakes).
\end{itemize}
For the noisy analysis of Theorem~\ref{thm:main}, we require that $\{\mu_{i,t}\}$ and $\{\pi_{s,t}\}$ are deterministic functions of the staleness buckets $\{\cB_{s,t}\}_{s=0}^{\tau}$ only (independent of the received model values).

The server then updates
\begin{equation}
 w_{t+1}
 := \sum_{i\in \cA_t} \mu_{i,t}\,\tilde w^{\mathrm{ul}}_{i,t}
 \;+
 \!\!\sum_{s:\,\cB_{s,t}=\emptyset}\!\! \pi_{s,t}\, w_{t-s}.
 \label{eq:update}
\end{equation}

\paragraph{Protocol requirements (minimal)}
The staleness-profile rule requires only that (i) the server can identify each realized \emph{total} staleness $s_{i,t}$ (e.g., clients report the global model version used, or the server tags broadcasts), and (ii) the server stores the last $\tau+1$ iterates $\{w_{t-s}\}_{s=0}^{\tau}$.

\paragraph{Convexity}
By construction all weights are nonnegative and
$
\sum_{i\in\cA_t}\mu_{i,t} + \sum_{s:\,\cB_{s,t}=\emptyset}\pi_{s,t}
= \sum_{s=0}^{\tau}\alpha_s = 1,
$
so \eqref{eq:update} is a convex combination.

\begin{algorithm}[t]
\caption{Noisy semi-asynchronous stale-IPM perceptron with server-enforced staleness profile (padding)}
\label{alg:main}
\begin{algorithmic}[1]
\STATE Initialize $w_0\leftarrow 0$; define $w_s\leftarrow 0$ for all integers $s<0$.
\STATE Fix staleness bounds $\tau_{\mathrm{dl}},\tau_{\mathrm{ul}}\in\mathbb{N}_0$ and set $\tau\leftarrow \tau_{\mathrm{dl}}+\tau_{\mathrm{ul}}$.
\STATE Fix target staleness profile $\bm{\alpha}=(\alpha_0,\ldots,\alpha_{\tau})$ with $\alpha_s\ge 0$, $\sum_{s=0}^{\tau}\alpha_s=1$.
\FOR{$t=0,1,2,\ldots$}
\STATE Server receives the set $\cA_t$ of client updates that arrive and are applied at round $t$.
\FORALL{$i\in\cA_t$ \textbf{in parallel}}
\STATE Server learns/infers $(s^{\mathrm{dl}}_{i,t},s^{\mathrm{ul}}_{i,t})$; set $a\leftarrow t-s^{\mathrm{ul}}_{i,t}$ and $s_{i,t}\leftarrow s^{\mathrm{dl}}_{i,t}+s^{\mathrm{ul}}_{i,t}$.
\STATE Client $i$ started local computation in round $a$ from $\tilde w^{\mathrm{dl}}_{i,t}=w_{a-s^{\mathrm{dl}}_{i,t}}+\delta_{i,t}$ and runs local perceptron on $T_i$ for at least one epoch, producing $(w_{i,t},k_{i,t})$.
\STATE Client returns $(w_{i,t},\,k_{i,t},\,s_{i,t})$; server receives $\tilde w^{\mathrm{ul}}_{i,t}=w_{i,t}+\xi_{i,t}$.
\ENDFOR
\FOR{$s=0,1,\ldots,\tau$}
\STATE $\cB_{s,t}\leftarrow\{i\in\cA_t:\ s_{i,t}=s\}$.
\IF{$\cB_{s,t}\neq\emptyset$}
\STATE Choose $\{\mu_{i,t}\}_{i\in \cB_{s,t}}$ with $\mu_{i,t}\ge 0$ and $\sum_{i\in \cB_{s,t}}\mu_{i,t}=\alpha_s$.
\STATE (E.g.) set $\mu_{i,t}=\alpha_s/|\cB_{s,t}|$ for all $i\in\cB_{s,t}$ and set $\pi_{s,t}\leftarrow 0$.
\ELSE
\STATE Set padding $\pi_{s,t}\leftarrow \alpha_s$ on cached vector $w_{t-s}$.
\ENDIF
\ENDFOR
\STATE Update $w_{t+1}$ by staleness-profile aggregation with padding:
\STATE $w_{t+1}\leftarrow \sum_{i\in\cA_t}\mu_{i,t}\tilde w^{\mathrm{ul}}_{i,t}
+ \sum_{s:\,\cB_{s,t}=\emptyset}\pi_{s,t}w_{t-s}$.
\ENDFOR
\end{algorithmic}
\end{algorithm}

\subsection{Noise model}

We model communication impairments via an effective additive noise abstraction with bounded conditional second moments \cite{wei2022}.
For each round $t$, let $H_t$ denote the algorithmic history up to the start of round $t$, and let
$\mathcal{F}_t := (H_t, \cA_t, \{s_{i,t}\}_{i\in \cA_t})$ denote the information available once the arrival set and
stalenesses in round $t$ are revealed.

\begin{definition}[Conditional zero-mean bounded-variance effective channel noise]
For every server round $t$ and every arriving client update $i\in\cA_t$,
\begin{align}
\E[\delta_{i,t}\mid \mathcal{F}_t] &= 0, &
\E\!\left[\norm{\delta_{i,t}}^2\mid \mathcal{F}_t\right] &\le \sigma^2_{\mathrm{dl}}, \label{eq:noise-down}\\
\E[\xi_{i,t}\mid \mathcal{F}_t, w_{i,t}] &= 0, &
\E\!\left[\norm{\xi_{i,t}}^2\mid \mathcal{F}_t, w_{i,t}\right] &\le \sigma^2_{\mathrm{ul}}. \label{eq:noise-up}
\end{align}
Define the total per-round noise-energy bound
\begin{equation}
V := \sigma^2_{\mathrm{dl}} + \sigma^2_{\mathrm{ul}}.
\label{eq:V}
\end{equation}
\end{definition}

\section{Expected Weighted Mistake Bound}

\subsection{Performance metric}

Let $k_{i,t}$ be the number of local perceptron mistake updates for the client model that arrives in round $t$. Define the weighted mistakes per round
\begin{equation}
\kappa_t := \sum_{i\in \cA_t} \mu_{i,t} \, k_{i,t},
\label{eq:kappa}
\end{equation}
and the cumulative weighted mistakes up to horizon $A$:
\begin{equation}
K_A := \sum_{t=0}^{A-1} \kappa_t, \qquad A\ge 1.
\label{eq:KA}
\end{equation}
(Any padding terms correspond to virtual participants with $0$ mistakes.)

\subsection{Main theorem}

\begin{theorem}[Expected bound under two-sided staleness and noise]
\label{thm:main}
Under margin separability \eqref{eq:separable}, $\norm{x}\le R$, the bounded-staleness model \eqref{eq:total-staleness}, the noisy communication model \eqref{eq:downlink}--\eqref{eq:noise-up}, and the staleness-profile server update \eqref{eq:update} (with $\tau=\tau_{\mathrm{dl}}+\tau_{\mathrm{ul}}$), define the mean total staleness under the profile
\begin{equation}
\bar s := \sum_{s=0}^{\tau} s\,\alpha_s,\qquad S:=1+\bar s.
\label{eq:S}
\end{equation}
Then for every horizon $A\ge 1$,
\begin{equation}
\E[K_A] \le \frac{S R^2}{\gamma^2} + \frac{\sqrt{S A V}}{\gamma}.
\label{eq:bound2} 
\end{equation}

In particular, when $V=0$ (no communication noise), \eqref{eq:bound2} reduces to $\E[K_A]\le S R^2/\gamma^2$ for all $A$.
\end{theorem}

\noindent As a corollary, \eqref{eq:bound2} implies $\E[K_A]/A \to 0$ as $A\to\infty$, i.e., a vanishing expected per-round weighted mistake rate.

\begin{remark}[Special cases and interpretation]
If $\tau=0$ then $S=1$ and \eqref{eq:bound2} becomes $\E[K_A]\le R^2/\gamma^2+\sqrt{A V}/\gamma$.
Theorem~\ref{thm:main} is stated in expectation; when $V=0$ the same argument gives the deterministic bound $K_A\le S R^2/\gamma^2$ for any realized participation/staleness schedule.
In particular, for $\tau=0$ this matches the standard IPM perceptron mistake bound \cite{mcdonald2010}, and for a single client it reduces to the classical perceptron bound \cite{novikoff1962}.
For $\tau>0$, staleness enters only through $S=1+\bar s\le \tau+1$; if $\alpha_0=1$ then rounds without a fresh update satisfy $w_{t+1}=w_t$.
\end{remark}

\section{Proof of Theorem~\ref{thm:main}}

We follow the standard perceptron ``progress versus norm growth'' template \cite{novikoff1962} and control staleness via a staleness-augmented Lyapunov/potential argument in the spirit of asynchronous optimization analyses \cite{peng2016}.

\subsection{Single-client inequalities}

Let $w^{(p)}_{i,t}$ denote client $i$'s local iterate after $p$ local mistake updates for the model that arrives in round $t$. Thus $w^{(0)}_{i,t}=\tilde w^{\mathrm{dl}}_{i,t}$ and for each mistake update on an example $(x^{(p)}_{i,t},y^{(p)}_{i,t})$ with $y^{(p)}_{i,t}\ip{w^{(p)}_{i,t}}{x^{(p)}_{i,t}}\le 0$,
\begin{equation}
 w^{(p+1)}_{i,t} = w^{(p)}_{i,t} + y^{(p)}_{i,t} x^{(p)}_{i,t}.
 \label{eq:local-update}
\end{equation}
The client returns $w_{i,t}:=w^{(k_{i,t})}_{i,t}$.

\begin{lemma}[Single-client perceptron progress (noisy stale initialization)]
\label{lem:single}
Fix a server round $t$ and $i\in \cA_t$. Under \eqref{eq:separable} and $\norm{x}\le R$,
\begin{align}
\ip{w^\star}{w_{i,t}} &\ge \ip{w^\star}{w_{t-s_{i,t}}} + \ip{w^\star}{\delta_{i,t}} + \gamma k_{i,t},
\label{eq:single-progress}\\
\norm{w_{i,t}}^2 &\le \norm{w_{t-s_{i,t}}+\delta_{i,t}}^2 + R^2 k_{i,t}.
\label{eq:single-norm}
\end{align}
\end{lemma}

\begin{proof}
For any mistake update \eqref{eq:local-update},
\begin{align*}
\ip{w^\star}{w^{(p+1)}_{i,t}}
&= \ip{w^\star}{w^{(p)}_{i,t}}
  + y^{(p)}_{i,t}\ip{w^\star}{x^{(p)}_{i,t}} \\
&\ge \ip{w^\star}{w^{(p)}_{i,t}} + \gamma,
\end{align*}

by \eqref{eq:separable}. Summing over $p=0,\ldots,k_{i,t}-1$ and using $w^{(0)}_{i,t}=w_{t-s_{i,t}}+\delta_{i,t}$ gives \eqref{eq:single-progress}.

Similarly,
\begin{align*}
\norm{w^{(p+1)}_{i,t}}^2
&= \norm{w^{(p)}_{i,t}}^2
 + 2 y^{(p)}_{i,t}\ip{w^{(p)}_{i,t}}{x^{(p)}_{i,t}}
 + \norm{x^{(p)}_{i,t}}^2 \\
&\le \norm{w^{(p)}_{i,t}}^2 + R^2,
\end{align*}

on a mistake, since $y^{(p)}_{i,t}\ip{w^{(p)}_{i,t}}{x^{(p)}_{i,t}}\le 0$ and $\norm{x^{(p)}_{i,t}}\le R$. Summing yields \eqref{eq:single-norm}.
\end{proof}

\subsection{Global recursions in expectation}

Define the scalar sequences
\begin{equation}
 a_t := \ip{w^\star}{w_t},\qquad b_t := \norm{w_t}^2.
 \label{eq:ab}
\end{equation}

\begin{lemma}[Stale-IPM inequalities under staleness-profile aggregation with padding]
\label{lem:global}
For every server round $t\ge 0$,
\begin{align}
\E[a_{t+1}] &\ge \sum_{s=0}^{\tau} \alpha_s\, \E[a_{t-s}] + \gamma\, \E[\kappa_t],
\label{eq:a-rec}\\
\E[b_{t+1}] &\le \sum_{s=0}^{\tau} \alpha_s\, \E[b_{t-s}] + R^2\, \E[\kappa_t] + V.
\label{eq:b-rec}
\end{align}
Here we use the convention $w_s\equiv 0$ for $s<0$, hence $a_s=b_s=0$ for $s<0$.
\end{lemma}

\begin{proof}
Fix a round $t$. During round $t$ a set $\cA_t$ of updates arrives, with realized stalenesses
$\{s_{i,t}\}_{i\in \cA_t}$ and buckets $\{\cB_{s,t}\}_{s=0}^\tau$.

By \eqref{eq:update},
\begin{equation}
 w_{t+1} = \sum_{i\in \cA_t} \mu_{i,t}(w_{i,t}+\xi_{i,t})
 \;+\!\!\sum_{s:\,\cB_{s,t}=\emptyset}\!\! \pi_{s,t} \, w_{t-s}.
 \label{eq:update-proof}
\end{equation}

\paragraph{Progress recursion}
Taking $\ip{w^\star}{\cdot}$ in \eqref{eq:update-proof} and applying Lemma~\ref{lem:single} gives
\begin{align*}
a_{t+1}
&\ge \sum_{i\in\cA_t}\mu_{i,t}\, a_{t-s_{i,t}} + \gamma\,\kappa_t \\
&\quad + \sum_{i\in\cA_t}\mu_{i,t}\ip{w^\star}{\delta_{i,t}}
+ \sum_{i\in\cA_t}\mu_{i,t}\ip{w^\star}{\xi_{i,t}}\\
&\quad + \sum_{s:\,\cB_{s,t}=\emptyset}\pi_{s,t} a_{t-s}.
\end{align*}

Condition on $\mathcal{F}_t$. Under this conditioning, the weights are fixed (they depend only on buckets) and the noise terms have zero conditional mean given $\mathcal{F}_t$ (for $\xi_{i,t}$ this follows from \eqref{eq:noise-up} by iterated expectation over $w_{i,t}$), hence their conditional expectations vanish. Finally, by construction of the profile, for any scalar sequence $\{Z_{t-s}\}_{s=0}^{\tau}$,
\begin{equation}
\sum_{i\in \cA_t}\mu_{i,t} Z_{t-s_{i,t}}
+\sum_{s:\,\cB_{s,t}=\emptyset}\pi_{s,t} Z_{t-s}
= \sum_{s=0}^{\tau} \alpha_s Z_{t-s}.
\label{eq:alpha-identity}
\end{equation}
Applying \eqref{eq:alpha-identity} with $Z_{t-s}=a_{t-s}$ and then using the law of total expectation yields \eqref{eq:a-rec}.

\paragraph{Norm recursion}
Since \eqref{eq:update-proof} is a convex combination, Jensen's inequality gives
\[
b_{t+1}=\norm{w_{t+1}}^2
\le \sum_{i\in\cA_t}\mu_{i,t}\norm{w_{i,t}+\xi_{i,t}}^2
+ \sum_{s:\,\cB_{s,t}=\emptyset}\pi_{s,t} b_{t-s}.
\]
Condition on $\mathcal{F}_t$. Using  \eqref{eq:noise-up},
\[
\E[\norm{w_{i,t}+\xi_{i,t}}^2\mid \mathcal{F}_t]
\le \E[\norm{w_{i,t}}^2\mid \mathcal{F}_t] + \sigma^2_{\mathrm{ul}}.
\]
By Lemma~\ref{lem:single}, $\norm{w_{i,t}}^2\le \norm{w_{t-s_{i,t}}+\delta_{i,t}}^2 + R^2 k_{i,t}$, and \eqref{eq:noise-down} implies
$\E[\norm{w_{t-s_{i,t}}+\delta_{i,t}}^2\mid \mathcal{F}_t]\le b_{t-s_{i,t}}+\sigma^2_{\mathrm{dl}}$.
Combining these bounds and substituting into Jensen yields
\begin{equation*} 
\begin{split}
    \mathbb{E}[b_{t+1} \mid \mathcal{F}_t] 
    &\leq \sum_{i\in\mathcal{A}_t}\mu_{i,t} b_{t-s_{i,t}} + \sum_{s:\,\mathcal{B}_{s,t}=\emptyset}\pi_{s,t} b_{t-s} \\
    &\quad + R^2\,\mathbb{E}[\kappa_t \mid \mathcal{F}_t] + V,
\end{split}
\end{equation*}
where $V=\sigma^2_{\mathrm{dl}}+\sigma^2_{\mathrm{ul}}$. Applying \eqref{eq:alpha-identity} with $Z_{t-s}=b_{t-s}$ and then using the law of total expectation yields \eqref{eq:b-rec}.

\end{proof}

\subsection{Lyapunov potential and conclusion}

Define tail sums $c_0:=1$ and for $j\in\{1,\ldots,\tau\}$,
\begin{equation}
 c_j := \sum_{s=j}^{\tau} \alpha_s.
 \label{eq:cj}
\end{equation}
Define the potentials for $t\ge 0$:
\begin{equation}
\Phi_t := \sum_{j=0}^{\tau} c_j\, \E[a_{t-j}],\qquad
\Psi_t := \sum_{j=0}^{\tau} c_j\, \E[b_{t-j}].
\label{eq:potentials}
\end{equation}
Using $c_j-c_{j+1}=\alpha_j$ (with $c_{\tau+1}:=0$), one can rewrite $\Phi_{t+1}=\Phi_t+\E[a_{t+1}]-\sum_{s=0}^{\tau}\alpha_s\E[a_{t-s}]$ (and similarly for $\Psi_t$); substituting Lemma~\ref{lem:global} yields the one-step bounds
\begin{equation}
\Phi_{t+1} \ge \Phi_t + \gamma\, \E[\kappa_t],\qquad
\Psi_{t+1} \le \Psi_t + R^2\, \E[\kappa_t] + V.
\label{eq:one-step}
\end{equation}
Telescoping from $t=0$ to $A-1$ gives
\begin{equation}
\Phi_A \ge \gamma\, \E[K_A],\qquad
\Psi_A \le R^2\, \E[K_A] + A V.
\label{eq:telescope}
\end{equation}

Let $Z:=\sum_{j=0}^{\tau} c_j w_{A-j}$. By linearity and $\norm{w^\star}=1$, $\Phi_A = \E[\ip{w^\star}{Z}] \le \E[\norm{Z}]$.
Moreover, using Jensen and Cauchy--Schwarz (with $c_j\ge 0$),
\begin{align*}
\E[\norm{Z}]^2
&\le \E[\norm{Z}^2] \\
&\le \left(\sum_{j=0}^{\tau} c_j\right)\,
\E\!\left[\sum_{j=0}^{\tau} c_j\,\norm{w_{A-j}}^2\right] \\
&= \left(\sum_{j=0}^{\tau} c_j\right)\Psi_A.
\end{align*}

A direct calculation yields
\begin{equation}
\sum_{j=0}^{\tau} c_j = 1+\sum_{s=0}^{\tau} s\,\alpha_s = S.
\label{eq:sumcj}
\end{equation}
Thus $\Phi_A^2 \le S\Psi_A$. Combining with \eqref{eq:telescope} gives
$\gamma^2 \E[K_A]^2 \le S(R^2 \E[K_A] + A V)$; solving for $\E[K_A]$ yields \eqref{eq:bound2}.

\section{Finite-Round Stabilization in the Noiseless Case}

\begin{figure*}[t]
\centering
\includegraphics[width=0.85\textwidth]{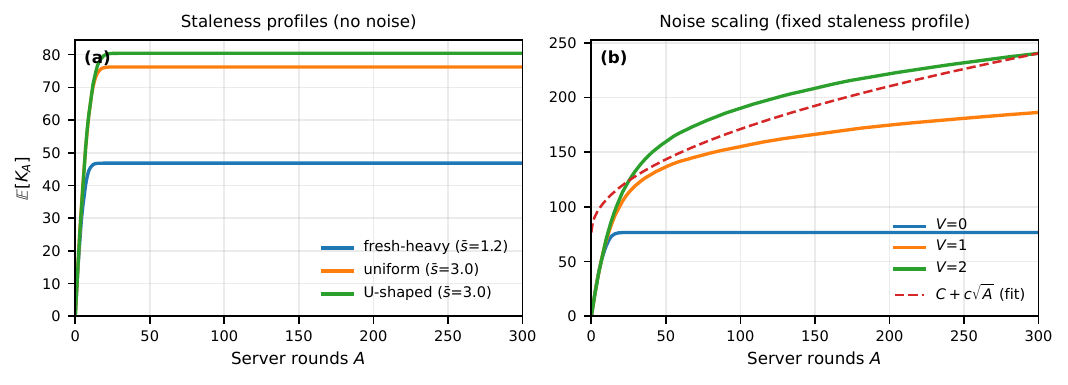}
\caption{Synthetic illustration of Theorem~\ref{thm:main} for Algorithm~\ref{alg:main}. (a) Noiseless links: different enforced staleness profiles $\bm{\alpha}$ (legend shows $\bar s$) lead to different cumulative weighted mistake levels. (b) Fixed profile: increasing noise energy $V$ increases $\E[K_A]$ approximately as $C+c\sqrt{A}$.}
\label{fig:staleness-noise}
\end{figure*}

Theorem~\ref{thm:main} controls cumulative (weighted) mistakes. In general, a finite mistake bound does not necessarily imply a finite-round bound in semi-asynchronous systems: counterexamples can exist.
We show that a finite-round stabilization bound \emph{does} follow in the noiseless case ($V=0$), under a mild liveness assumption and a simple within-bucket weighting property.

Throughout this section we specialize to noiseless links:
\begin{equation}
\delta_{i,t}\equiv 0,\qquad \xi_{i,t}\equiv 0,
\label{eq:noiseless}
\end{equation}
and define $K_\infty := \sum_{t=0}^{\infty} \kappa_t$.

\begin{theorem}[Finite-round stabilization in the noiseless case ($V=0$)]
\label{thm:rounds}
Assume \eqref{eq:noiseless} and $\alpha_0>0$, and that every participating client executes at least one complete pass over its local dataset $T_i$ in Algorithm~\ref{alg:main}.
Assume further:
\begin{itemize}
\item[(L1)] (\emph{Uniform fresh participation}) There exists $p_{\min}>0$ such that for every client $i$ and every round $t$,
\[
\mathbb{P}\big(i\in\cA_t,\ s_{i,t}=0\ \big|\ H_t\big)\ \ge\ p_{\min}
\]
\end{itemize}
The server uses any within-bucket rule (in the noiseless setting this may depend on the reported mistake counts $\{k_{i,t}\}$) such that whenever some fresh client makes a mistake in round $t$,
\[
\sum_{i\in \cB_{0,t}}\mu_{i,t}k_{i,t}\ \ge\ \alpha_0.
\]
Define
\begin{equation}
\begin{split}
T_{\mathrm{hit}} &:= \inf\{t\ge 0:\ y\ip{w_t}{x}>0,\ \forall (x,y)\in T\},\\
T_{\mathrm{stab}} &:= \inf\{t\ge 0:\ y\ip{w_r}{x}>0,\ \forall (x,y)\in T,\ \forall r\ge t\}.
\end{split}
\label{eq:ThitTstab}
\end{equation}
Then
\begin{equation}\label{eq:hit-stab-bounds}
\begin{aligned}
\mathbb{E}[T_{\mathrm{hit}}] &\le \frac{S R^2}{\alpha_0 p_{\min}\gamma^2}, \\
\mathbb{E}[T_{\mathrm{stab}}] &\le (\tau+1)\frac{S R^2}{\alpha_0 p_{\min}\gamma^2}.
\end{aligned}
\end{equation}

\end{theorem}

\begin{remark}[Noisy vs.\ noiseless within-bucket weighting]
\label{rem:noisy-vs-noiseless-weights}
In Theorem~\ref{thm:main} (noisy links), the bucket-only determinism of $\{\mu_{i,t}\}$ is used in Lemma~\ref{lem:global} so that, conditioning on $\mathcal{F}_t$, the weights are fixed and the zero-mean noise terms vanish in expectation. Under \eqref{eq:noiseless} (i.e., $V=0$), these noise terms are identically zero, so within-bucket weights may be chosen in any (possibly mistake-aware) way to enforce the fresh-mistake weight condition; e.g., put the fresh mass $\alpha_0$ uniformly on $\{i\in\cB_{0,t}:k_{i,t}>0\}$ (and zero on the rest).
\end{remark}

\begin{proof}
Work in the noiseless regime \eqref{eq:noiseless}. For each round $t$, define the event
\[
\mathcal{E}_t := \{\, y\ip{w_t}{x}\le 0 \text{ for some }(x,y)\in T \,\},
\]
i.e., $\mathcal{E}_t$ is the event that $w_t$ is \emph{not} globally correct.

\paragraph{Step 1: A one-round lower bound on $\E[\kappa_t\mid H_t]$.}
Fix $t\ge 0$ and condition on $H_t$.

If $\mathcal{E}_t$ does not occur then $\mathbf{1}\{\mathcal{E}_t\}=0$ and the inequality below is trivial.
Otherwise, since $T=\biguplus_{i=1}^m T_i$, there exists a client index
$i_t\in\{1,\ldots,m\}$ and some $(x,y)\in T_{i_t}$ such that $y\ip{w_t}{x}\le 0$.

Consider the fresh-participation event
\[
G_t := \{\, i_t\in\cA_t,\ s_{i_t,t}=0 \,\}.
\]
By (L1), $\mathbb{P}(G_t\mid H_t)\ge p_{\min}$.
On $G_t$, client $i_t$ starts from $w_t$ (noiseless downlink) and executes at least one epoch over $T_{i_t}$.
Because $w_t$ misclassifies at least one example in $T_{i_t}$, the first epoch must contain at least one perceptron mistake update, hence $k_{i_t,t}\ge 1$.
Therefore, on $G_t$ there exists a fresh client that makes a mistake in round $t$, and by the fresh-mistake weight condition we have
\[
\kappa_t
=\sum_{i\in\cA_t}\mu_{i,t}k_{i,t}
\ge \sum_{i\in \cB_{0,t}}\mu_{i,t}k_{i,t}
\ge \alpha_0.
\]
Consequently,
\[
\E[\kappa_t\mid H_t]
\ge \alpha_0 p_{\min}\,\mathbf{1}\{\mathcal{E}_t\}.
\]

\paragraph{Step 2: Bounding the expected number of incorrect rounds.}
Taking expectations and summing over $t\ge 0$ gives
\[
\E[K_\infty]
= \sum_{t=0}^\infty \E[\kappa_t]
= \sum_{t=0}^\infty \E\big[\E[\kappa_t\mid H_t]\big]
\ge \alpha_0 p_{\min}\sum_{t=0}^\infty \mathbb{P}(\mathcal{E}_t).
\]
By Theorem~1 with $V=0$, $\mathbb{E}[K_A]\le SR^2/\gamma^2$ for all $A$, hence
$\mathbb{E}[K_\infty]\le SR^2/\gamma^2$. Therefore
\[
\sum_{t=0}^\infty \mathbb{P}(\mathcal{E}_t)
 \le \frac{SR^2}{\alpha_0 p_{\min}\gamma^2}.
\]

\paragraph{Bound for $T_{\mathrm{hit}}$.}
By definition of $T_{\mathrm{hit}}$ in \eqref{eq:ThitTstab}, for every sample path and every $t<T_{\mathrm{hit}}$ the iterate $w_t$ is not globally correct, i.e., $\mathcal{E}_t$ occurs. Hence
\[
T_{\mathrm{hit}}
= \sum_{t=0}^\infty \mathbf{1}\{t<T_{\mathrm{hit}}\}
\le \sum_{t=0}^\infty \mathbf{1}\{\mathcal{E}_t\}.
\]
Taking expectations yields
\[
\mathbb{E}[T_{\mathrm{hit}}] \le \sum_{t=0}^\infty \mathbb{P}(\mathcal{E}_t)
\le \frac{SR^2}{\alpha_0 p_{\min}\gamma^2}.
\]

\paragraph{A window condition implying stabilization.}
Suppose that for some round $t$ we have $y\ip{w_{t-j}}{x}>0$ for all $(x,y)\in T$ and all $j=0,1,\ldots,\tau$.
Consider any client update applied at round $t$: its total staleness satisfies $s_{i,t}\le\tau$, so the client starts from $w_{t-s_{i,t}}$, which is globally correct and in particular correct on $T_i$.
Under Algorithm~\ref{alg:main}, this implies the client makes no local mistake updates and returns $w_{i,t}=w_{t-s_{i,t}}$.
Padding terms in \eqref{eq:update} are also among $\{w_{t-j}\}_{j=0}^{\tau}$.
Therefore $w_{t+1}$ is a convex combination of globally correct vectors.
Since the set $\{w:\ y\ip{w}{x}>0,\ \forall (x,y)\in T\}$ is an intersection of (open) halfspaces, it is convex, hence $w_{t+1}$ is also globally correct.
Repeating the same argument inductively shows $w_r$ is globally correct for all $r\ge t$, i.e., $T_{\mathrm{stab}}\le t$.

Equivalently, if $t<T_{\mathrm{stab}}$ then the window $\{w_{t-j}\}_{j=0}^{\tau}$ cannot be entirely correct, so at least one of $w_t,w_{t-1},\ldots,w_{t-\tau}$ is incorrect. Thus, with the convention $\mathbf{1}\{\mathcal{E}_s\}=0$ for $s<0$,
\[
\mathbf{1}\{t<T_{\mathrm{stab}}\} \le \sum_{j=0}^{\tau} \mathbf{1}\{\mathcal{E}_{t-j}\}.
\]
Summing over $t\ge 0$ gives
\begin{align*}
T_{\mathrm{stab}}
&= \sum_{t=0}^\infty \mathbf{1}\{t<T_{\mathrm{stab}}\} \\
&\le \sum_{t=0}^\infty \sum_{j=0}^{\tau}\mathbf{1}\{\mathcal{E}_{t-j}\} \\
&= (\tau+1)\sum_{t=0}^\infty \mathbf{1}\{\mathcal{E}_t\}.
\end{align*}
Taking expectations and using the bound on $\sum_{t\ge 0}\mathbb{P}(\mathcal{E}_t)$ yields
\[
\mathbb{E}[T_{\mathrm{stab}}] \le (\tau+1)\sum_{t=0}^\infty \mathbb{P}(E_t)
\le (\tau+1)\frac{SR^2}{\alpha_0 p_{\min}\gamma^2}.
\]

This establishes (31). 
\end{proof}

\section{Experiments}
\label{sec:experiments}

We provide a synthetic experiment to illustrate the qualitative behavior predicted by Theorem~\ref{thm:main}. We generate linearly separable data satisfying Definition~1, partition it across clients, and run Algorithm~\ref{alg:main} for $A$ server rounds under partial participation and bounded total staleness $\tau=\tau_{\mathrm{dl}}+\tau_{\mathrm{ul}}$. For illustration we use a simple randomized participation/delay schedule (the theory does not require any such model) and evaluate several enforced profiles $\bm{\alpha}$; Fig.~\ref{fig:staleness-noise} reports Monte Carlo estimates of $\E[K_A]$.
In the noiseless case ($V=0$), panel~(a) shows that emphasizing fresher updates (smaller enforced mean staleness $\bar s$) yields a smaller mistake budget, consistent with the $S=1+\bar s$ dependence in \eqref{eq:bound2}. Panel~(b) fixes $\bm{\alpha}$ and varies the total noise energy $V$ (implemented via zero-mean additive perturbations with bounded second moment), illustrating the predicted $O(\sqrt{A})$ growth in \eqref{eq:bound2}.

\section{Discussion}
\label{sec:discussion}

Algorithm~\ref{alg:main} implements a deterministic \emph{age-mixing} rule: in every round the server forms a convex combination whose total-staleness distribution matches a chosen profile $\bm{\alpha}$, padding missing buckets with cached iterates. This removes the need to model delays or participation and explains Theorem~\ref{thm:main}: asynchrony enters only through the enforced mean staleness $\bar s$ (via $S=1+\bar s$), while link noise contributes through $V$ in \eqref{eq:bound2}. Since $S$ appears linearly in $SR^2/\gamma^2$ and inside the square root in $\sqrt{SAV}/\gamma$, reducing $\bar s$ improves both the noiseless plateau and the noisy $\sqrt{A}$ slope.

A practical profile-design heuristic is to estimate bucket-occupancy frequencies online and choose $\bm{\alpha}$ supported on stalenesses that are reliably present; otherwise that mass is frequently assigned to cached iterates, which is safe but reduces the fraction of new information mixed each round. Among feasible profiles, \eqref{eq:bound2} suggests minimizing $\bar s$ to shrink both terms. Finally, \eqref{eq:bound2} also reveals a noise-limited regime: once $\sqrt{SAV}/\gamma$ dominates $SR^2/\gamma^2$, additional rounds yield diminishing returns unless the effective $V$ is reduced.

\section{Conclusion}

We analyzed a semi-asynchronous distributed perceptron with bounded downlink staleness, bounded uplink/computation turnaround delay, partial participation, and additive uplink/downlink communication noise. The finite-horizon guarantee in Theorem~\ref{thm:main} makes the dependence on the \emph{server-enforced} mean total staleness and on noise energy explicit, and recovers standard perceptron \cite{novikoff1962} and IPM \cite{mcdonald2010} results as special cases. In addition, in the noiseless regime, Theorem~\ref{thm:rounds} yields a finite-round stabilization guarantee under a mild fresh-participation condition.

\end{document}